\def\year{2021}\relax
\documentclass[letterpaper]{article} 
\usepackage{aaai21}  
\usepackage{times}  
\usepackage{helvet} 
\usepackage{courier}  
\usepackage[hyphens]{url}  
\usepackage{graphicx} 
\usepackage{natbib}  
\usepackage{caption} 
\usepackage{amsfonts}
\usepackage{amsmath}
\usepackage{amsthm}
\usepackage{amssymb}
\usepackage{mathtools}
\usepackage{appendix}

\newtheorem{theorem}{Theorem}

\newtheorem{lemma}{Lemma}

\newcommand{\norm}[1]{\left\lVert#1\right\rVert}
\newcommand{\safeSet}{\mathcal{S}}
\newcommand{\decisionSet}{\mathcal{X}}
\newcommand{\actionSet}{A}
\newcommand{\reals}{\mathbb{R}}

\title{Optimistic and Adaptive Lagrangian Hedging}
\author{
    Ryan D'Orazio,\textsuperscript{\rm 1}\\
    Ruitong Huang \textsuperscript{\rm 2}\\
}
\affiliations {
    \textsuperscript{\rm 1} Universit\'e de Montr\'eal and MILA \thanks{This work is partially done when Ryan D'Orazio was working at Borealis AI.}\\
    \textsuperscript{\rm 2} Borealis AI \\
    ryan.dorazio@mila.quebec, 
    rtonghuang@gmail.com
}
\begin{document}

\maketitle

\begin{abstract}
In online learning an algorithm plays against an 
environment with losses possibly picked by
an adversary at each round.
The generality of this framework includes 
problems that are not adversarial, for example
offline optimization, or saddle point problems (i.e. min max optimization).
However, online algorithms are typically not designed 
to leverage additional structure present in non-adversarial problems.
Recently, slight modifications to well-known online 
algorithms such as optimism and adaptive step sizes
have been used
in several domains to accelerate online learning -- 
recovering optimal rates in offline smooth optimization,
and accelerating convergence to saddle points or social
welfare in smooth games.
In this work we introduce optimism and adaptive stepsizes
to Lagrangian hedging, a class of online algorithms
that includes regret-matching, and hedge (i.e. multiplicative weights). 
Our results include: a general general regret bound; a path length regret bound for a fixed
smooth loss, applicable to an optimistic variant of regret-matching and regret-matching+; optimistic regret bounds for $\Phi$ regret,
a framework that includes external, internal, and swap regret;
and optimistic bounds for a family of algorithms that includes
regret-matching+ as a special case.

\end{abstract}

\section{Introduction}

Online optimization is a general framework applicable to various problems such as offline optimization, 
and finding equilibria in games. 
Typical algorithms only use first-order information~(i.e. a subgradient or gradient),
such as online mirror descent (MD)~\citep{nemirovsky1983problem, warmuth1997continuous, beck2003mirror} which generalizes projected gradient descent~(see for example ~\cite{orabona2019modern}), 
and follow the regularized leader (FTRL)~\citep{shalev2006online,abernethy2009competing, nesterov2009primal}.\footnote{See \citeauthor{orabona2019modern} for an excellent historical overview of MD and FTRL.}

In general, online learning is adversarial, losses
may change almost arbitrarily from one time step to the
next. However, most problems of interest including
offline optimization, and saddle point optimization
can be ``predictable." That is, the sequence of losses 
induced by running an online algorithm in these settings
has specific structure and can be predictable under the right
conditions, like smoothness (i.e. Lipschitz continuous gradient).
When losses are predictable a powerful framework is 
optimistic online learning~\citep{Rakhlin13a,Rakhlin13b,chiang2012online}.
Where algorithms are modified to incorporate a \emph{guess} of the next loss, $m_t$, into their update.

Combining optimism with MD and FTRL yields their
optimistic counterparts, optimistic mirror descent (OMD),
and optimistic follow the regularized leader (OFTRL), respectively.
OMD and OFTRL both provide tangible benefits when problems are not quite adversarial. For example,
faster convergence to a saddle point on the average~\cite{Rakhlin13b,syrgkanis2015fast,farina2019stable,farina2019optimistic}; faster convergence to optimal social wellfare in n-player games~\cite{syrgkanis2015fast}; last iterate convergence
in games~\citep{daskalakis2018last};
acceleration in offline or online optimization~\citep{cutkosky2019anytime,mohri2016accelerating,joulanisimpler, joulani2017modular}.

Interestingly, much of the analysis of optimistic algorithms is black-box. For example, most of the
results rely on regret bounds being of a particular form,
which is satisfied by both OMD and OFTRL. Naturally,
one may ask what other classes of algorithms can be 
combined with optimism to achieve faster rates in 
predictable problems?

In this paper we extend the idea of optimism to
the class of algorithms known as Lagrangian hedging~\cite{gordon2007no}.
Unfortunately, the regret bounds attained are not
consistent with OMD and OFTRL, therefore, immediate
theoretical acceleration via the previously mentioned works is not attained. 
However, our analysis provides interesting regret
bounds that should be small given a ``good" guess.
And in the case for a smooth fixed loss we show a
path length bound for the regret. 
This result, for example, is applicable to an optimistic varaint of the 
well-known regret-matching algorithm
when used to train a linear regressor with $L_1$ regularization and the least-squares loss~\cite{schuurmans2016deep}.

Additionally, our analysis extends beyond the typical
regret objectives of MD and FTRL, and includes regret
bounds for internal and swap regret~\cite{cesa2006prediction}.
To the best of our knowledge, our results provide
the first optimistic and adaptive algorithms for 
minimizing internal regret, with
possible applications including
finding correlated equilibria in $n$-player general sum games~\cite{cesa2006prediction}.

\section{Background}

\subsection{Online Linear Optimization}

In online convex optimization an algorithm
$\mathcal{A}$ interacts with an environment 
for $T$ rounds~\cite{GIGA}. 
In each round $t$, $\mathcal{A}$ selects an
iterate $x_t$ within some convex compact 
set $\decisionSet$, afterwhich a convex loss function
$\ell_t : {\decisionSet} \to \reals$ chosen by the environment is revealed. 
Furthermore, $\mathcal{A}$ is only allowed to 
use information from previous rounds.
The performance of $\mathcal{A}$ after $T$ rounds
is measured by its regret
\begin{align}
    R^T_\decisionSet = \sum_{t=1}^T \ell_t(x_t) -\underset{x \in \decisionSet}{\min} \sum_{t=1}^T \ell_t (x).
\end{align}
The objective is to ensure sublinear regret,
$R^T_\decisionSet \in o(T)$, e.g $R^T_\decisionSet \in
O(\sqrt{T})$.
In the most general of settings, no assumptions
are made on the sequence of losses $\{\ell_t\}_{t\leq T}$, they may 
be chosen by an adversary with knowledge of $\mathcal{A}$.

If each loss $\ell_t$ is subdifferentiable at 
$x_t$, 
then there exists
a vector $\partial \ell_t(x_t)$ (a subgradient) such that
\[
    \ell_t(x) \geq \ell_t(x_t) + \langle \partial \ell_t(x_t), x-x_t \rangle \quad \forall x \in \decisionSet.
\]
Provided $\mathcal{A}$ has access to a subgradient,
it is enough to design algorithms for linear losses.
The original regret $R^T_\decisionSet$ is upper bounded 
by the regret with respect to the linear losses $\{\Tilde{\ell}_t\}_{t\leq T}$, where
$\Tilde{\ell}_t(x) = \langle \partial \ell_t(x_t), x \rangle$.
For the remainder of the paper we assume linear losses
unless specified otherwise.

\subsection{Lagrangian Hedging}

Lagrangian hedging defines a class of algorithms
for online linear optimization~\cite{gordon2007no}.
The class generalizes potential based methods
introduced by \citeauthor{cesa2003potential}
for learning with expert advice~\cite{cesa2003potential},\footnote{Learning
with expert advice resembles online linear optimization
where the decision set $\decisionSet$ is an $n$-dimensional simplex $\Delta^n$, which is interpreted
as the set of distributions over $n$ experts.}
and includes the well-known Hedge aglorithm (also known as multiplicative weights)~\cite{freund1997decision},
and regret-matching~\cite{Hart00Rm} .

At each round $t$ a Lagrangian hedging algorithm 
maintains a regret vector 
\[
s_{1:t-1} = s_{1:t-2} + \langle \ell_{t-1}, x_{t-1}\rangle u - \ell_{t-1},
\]
with the initial vector initialized as $ s_{1:0} = s_0 = 0$.
The change in the regret
vector is denoted as $s_t = \langle \ell_t, x_t \rangle u - \ell_t$, $s_{1:t} = \sum_{k=1}^t s_k$.
$u$ is a vector such that $\langle u, x\rangle = 1$ for any
$x \in \decisionSet$.
As mentioned by \citeauthor{gordon2007no}~\cite{gordon2007no},
if no such $u$ can be found then we may append an extra
$1$ for each $x \in \decisionSet$. 
Then we can take $u$ to be
the vector of zeros except for a $1$ coinciding with the new dimension added to $x$, as well as append a $0$ to each loss. $s_{1:t}$ is referred to
as the regret vector because it tracks how well
an algorithm has done so far,
\[
    \sum_{t=1}^T \langle \ell_t, x_t\rangle
    - \sum_{t=1}^T \langle \ell_t, x \rangle
    = \langle s_{1:T} , x \rangle \quad \forall x \in  \decisionSet.
\]
The regret is then simply $R^T_\decisionSet = \max_{x \in \decisionSet} \langle s_{1:T},x \rangle$.

Instead of explicitly ensuring the regret to be small, 
Lagrangian hedging ensures $s_{1:T}$ is not too far from 
a safe set $\safeSet$. The safe set is defined
to be the polar cone to $\decisionSet$,
\[
    \safeSet = \{ s :\forall x \in\decisionSet \, \langle s, x \rangle \leq 0
    \}.
\]
Forcing $s_{1:T}$ to be in $\safeSet$ may not be possible,
as it would guarantee $R^T_\decisionSet \leq 0$ when
it is possible to encounter an adversary that guarantees
$\Omega(\sqrt{T})$ regret~\citep{orabona2019modern, hazan2016introduction}.
However, 
\begin{align}
    R^T_\decisionSet &= \max_{x \in \decisionSet} \langle s_{1:T},x \rangle  \leq \max_{x \in \decisionSet} \langle s_{1:T} - s,x \rangle \quad \forall s \in \safeSet \nonumber\\
    & \leq \inf_{s\in \safeSet} \norm{s_{1:T} - s}\max_{x \in \decisionSet}\norm{x}_\ast. \label{reg-dist}
\end{align}
Therefore, if the distance of $s_T$ to the set 
$\safeSet$ grows at a sublinear rate then
the regret will be sublinear, since by assumption
the set $\decisionSet$ is bounded, $\norm{x}_\ast \leq D$. The norm $\norm{\cdot}_\ast$ is the dual norm of $\norm{\cdot}$,
defined as $\norm{x}_\ast = \sup \{\langle x, y\rangle | \norm{y} \leq 1 \}$.

Additionally, we assume the change in the regret
vector is bounded in norm,
$\norm{s_t}^2 \leq C$. 
This assumption is similar to assuming 
bounded linear functions $\norm{\ell_t} \leq C$, or in the convex (possibly non-linear) case $\norm{\partial \ell_t } \leq C$ (i.e convex Lipschitz continuous functions).

The distance of $s_{1:t}$ to $\safeSet$ is then controlled
via a smooth potential function $F$, with the following conditions:

\begin{align}
    &F(s) \leq 0 \quad \forall s \in \safeSet \label{safe}\\
    &F(x+ y) \leq F(x) + \langle \partial F(x), y \rangle + \frac{L}{2}\norm{y}^2  \label{smooth}\\
    &(F(s) + A)^+ \geq \inf_{s' \in S}B \norm{s-s'}^p,   \label{dist-bound}
\end{align}
for constants $L,B > 0$, $A \geq 0$, and $1\leq p \leq 2$.
$\partial F(x)$ is a subgradient of $F$ at $x$.
$(x)^+$ refers to the Relu operation which sets
all negative values in the vector to $0$.
In addition to the above conditions we will also assume 
that $F$ is convex, and therefore differentiable
with $\partial F(x) = \nabla F(x)$, the gradient of $F$ at $x$.
When $F$ is differentiable condition (\ref{smooth})
is equivalent to Lipschitz continuity of the gradient,
$\norm{\nabla F(x) -\nabla F(y)}_\ast \leq L \norm{x-y}$~\cite{nesterov2018lectures}[Theorem 2.1.5].

Once an appropriate potential function is chosen,
a Lagrangian hedging algorithm ensures $F(s_{1:T}) \in O(T)$
by picking an iterate at each round $t$ such that
\begin{align}
    \langle \nabla F(s_{1:t-1}), s_t \rangle \leq 0,\label{blackwell}
\end{align}
for any possible $s_t$ ($s_t$ can change depending on the 
loss $\ell_t$ picked by the environment).
The above inequality is also known as the Blackwell
condition, often used in potential based expert algorithms~\cite{cesa2006prediction}, 
and, as shown by \citeauthor{gordon2007no},
is guaranteed if the iterate at time $t$ is chosen by the following rule
\begin{align} \label{LH}
   x_t =  \begin{cases} 
      \frac{\nabla F(s_{1:t-1})}{\langle \nabla F(s_{1:t-1}), u \rangle } & \mbox{if } \langle \nabla F(s_{1:t-1}), u \rangle > 0 \\
      \mbox{arbitrary } x  \in \decisionSet & o.w.
   \end{cases}
\end{align}
\citeauthor{gordon2007no} also showed that procedure (\ref{LH})
always yields a feasible iterate $x_t \in \decisionSet$.

Equipped with the Blackwell condition
and the smoothness of $F$ (condition \ref{smooth}),
the growth of $F(s_{1:t})$ is easily bounded by
\begin{align*}
    F(s_{1:t}) = F(s_{1:t-1}+s_t)& \leq F(s_{1:t-1}) + \frac{L}{2}\norm{s_t}^2  \\
    &\leq F(s_{1:t-1}) + \frac{LC}{2}.
\end{align*}
Summing across time and with $s_{1:0} = s_0 = 0$,
\[
    F(s_{1:t}) \leq F(0) +\frac{LCt}{2} \leq \frac{LCt}{2},
\]
since $0 \in \safeSet$. With a linear bound on 
$F(s_{1:t})$ a regret bound follows immediately
by condition (\ref{dist-bound}) and (\ref{reg-dist}),

\begin{align}
        R^T_{\decisionSet} \leq D\left(\frac{LCT+2A}{2B}\right)^{1/p}. \label{lh-regret}
\end{align}

If $p=1$ then the regret is linear, however,
as mentioned by \citeauthor{gordon2007no},
a stepsize can be used to achieve sublinear regret.
We can define a new potential function with stepsize $\eta$,
$F_\eta(s) =F(\eta s)$.
The smoothness condition becomes
\begin{align*}
        F_\eta(x+y) &= F(\eta (x+y)) \leq \\
        &F(\eta x) + \eta \langle \nabla F(\eta x), y \rangle  + \frac{\eta^2L}{2}\norm{y}^2, \\
        & F_\eta( x) + \langle \nabla F_\eta(x), y \rangle  + \frac{\eta^2L}{2}\norm{y}^2.
\end{align*}

$F_\eta$ is therefore a valid potential function,
with condition (4) now being
\[
    (F_\eta(s) + A)^+ \geq \inf_{s' \in S}\eta^p B \norm{s-s'}^p.\footnote{See \citeauthor{gordon2007no} for details.}
\]
Following the same arguments
as \citeauthor{gordon2007no}~(\cite{gordon2007no}[Theorem 3]), the
regret becomes
\begin{align}
     R^T_{\decisionSet} \leq D\left(\frac{\eta^{2} LCT+2A}{2B\eta^p}\right)^{1/p}. \label{lh-step-regret}
\end{align}
When $p=1$ a stepsize $\eta \in O(\frac{1}{\sqrt{T}})$ achieves a regret 
bound of $O(\sqrt{T})$, similar to the standard results in  
MD and FTRL analysis (see \cite{orabona2019modern} for
example). 
Despite this guarantee on regret, this bound does not hold uniformly 
over time, one must have knowledge of the horizon $T$ to select 
a stepsize. 
However, the standard doubling trick can be applied to achieve 
a regret guarantee for all time steps, requiring
algorithm resets after exponentially growing time
intervals~\cite{cesa2006prediction}.

In MD, FTRL, and the potential based
approaches from expert problems, however, a stepsize schedule of
$\eta_t \in O(\frac{1}{\sqrt{t}}))$ is enough to achieve a $O(\sqrt{T})$ regret bound
that holds uniformly over time (applies to any time horizon). 
Given that Lagrangian hedging generalizes potential based methods,
a similar result likely should hold. 
Indeed we show with the help of the following simple yet important  lemma,
that the same learning rate schedule would suffice for
Lagrangian Hedging algorithms with potential functions that need a learning rate (i.e $p=1$). 
This result is interesting as it makes no
additional assumptions on the potential function;
whereas, for example when viewing multiplicative
weights as a potential based method and a 
specific instance of Lagrangian hedging,
inequalities 
particular to the algorithm (a specific potential function) are used to derive the regret bounds that hold uniformly over time~\cite{cesa2006prediction}.

First we extend the Lagrangian hedging framework with
an arbitrary sequence of stepsizes $\{ \eta_t\}_{t\leq T}$, where the potential function
$F(\eta_t s)$ is used at round $t$ to construct the
iterate $x_t$,
\begin{align}
   x_t =  \begin{cases} 
      \frac{\nabla F(\eta_t s_{1:t-1})}{\langle \nabla F(\eta_t s_{1:t-1}), u \rangle } & \mbox{if } \langle \nabla F(\eta_t s_{1:t-1}), u \rangle > 0 \\
      \mbox{arbitrary } x  \in \decisionSet & o.w.
   \end{cases} \label{lh-step}
\end{align}

\begin{lemma}
Assume $F$ is a convex function 
satisfying condition (\ref{safe}), 
consider step sizes $0 < \eta_t \leq \eta_{t-1}$,
then
\[
    F(\eta_t s) \leq \frac{\eta_t}{\eta_{t-1}}F(\eta_{t-1} s).
\]
\end{lemma}
\begin{proof}
\begin{align*}
    F(\eta_t s) &= F\left(\frac{\eta_t}{\eta_{t-1}} \eta_{t-1}s + 0\right) = F\left(\frac{\eta_t}{\eta_{t-1}} \eta_{t-1}s + \left(1 - \frac{\eta_t}{\eta_{t-1}}\right)0\right) \\
    &\leq \frac{\eta_t}{\eta_{t-1}} F(\eta_{t-1}s) + (1 - \frac{\eta_t}{\eta_{t-1}})F(0) \\
    &\leq \frac{\eta_t}{\eta_{t-1}} F(\eta_{t-1}s) \mbox{, \quad since $0 \in \safeSet$}.
\end{align*}

\end{proof}

Coupling the above lemma with the  algorithm
(\ref{lh-step}) and the Blackwell condition 
yields a bound on the growth of $F(\eta_t s_t)$ 
and therefore a regret bound. 
Such a bound is a special case of optimistic
Lagrangian hedging when the prediction is $0$, and so we defer the presentation to the next section.



\section{Adaptivity and Optimism in Lagrangian
Hedging}

In this section we present the optimistic 
Lagrangian hedging algorithm along with
adaptive stepsizes and the regret guarantees.
Optimistic Lagangian hedging leverages a prediction $m_t$ at round $t$ to construct the
iterate $x_t$. In the optimistic and adaptive variants of
MD and FTRL, one hopes to have 
$m_t \approx \ell_t$ since the regret bounds
attained are usually of the form
$O\left(\sqrt{\sum_{t=1}^T \norm{m_t -\ell_t}^2}\right)$, with adaptive stepsizes (in the case of MD)
similar to 
\begin{align}
      \eta_t = \frac{1}{\sqrt{\sum_{s=1}^{t-1}\norm{m_s - \ell_s}^2}}.\label{ada-MD}
\end{align}

In optimistic Lagrangian hedging
we hope the prediction $m_t$ to be a good predictor of the change in the regret vector
$m_t \approx s_t$, 
with the provable regret bound of $O\left(\sqrt{\sum_{t=1}^T\norm{m_t -s_t}^2} \right)$.
Interestingly for the case of $p=2$ no adaptive
step size is needed!

\subsection{General Optimistic Bound}

Given a prediction $m_t$ we define
optimistic Lagrangian hedging with
stepsizes $\eta_t$ as the following rule

\begin{align}
\label{lh-opt-step}
   x_t =  \begin{cases} 
      \frac{\nabla F(\eta_t (s_{1:t-1} +m_t) )}{\langle \nabla F(\eta_t (s_{1:t-1}+m_t)), u \rangle } & \mbox{if } \langle \nabla F(\eta_t (s_{1:t-1} +m_t)), u \rangle > 0 \\
      \mbox{arbitrary } x  \in \decisionSet & o.w.
   \end{cases} 
\end{align}

Optimistic Lagrangian hedging then 
guarantees the general upper bound on the 
growth of the potential function.

\begin{theorem} \label{thm-opt-step}
An optimistic Lagrangian hedging algorithm with
a convex potential function $F$ satisfying conditions
(3-4) and positive decreasing stepsizes $0 < \eta_t \leq \eta_{t-1}$, ensures 
\[
   F(\eta_T s_{1:T}) \leq \frac{L}{2} \sum_{t=1}^T\eta_T \eta_t \norm{s_t-m_t}^2.
\]

\end{theorem}
\begin{proof}
 From the same arguments as Gordon, we have the following
 Blackwell condition
 \[
    \langle \nabla F(\eta_t (s_{1:t-1}+m_t)), s_t) \rangle \leq 0.
 \]
 By the smoothness of $F$ we have 
\begin{align*}
&F(\eta_t s_{1:t}) = F(\eta_t (s_{t-1}+m_t + s_t - m_t)) \leq \\
&F(\eta_t (s_{1:t-1}+m_t)) \\
&+ \langle \nabla F(\eta_t (s_{1:t-1}+m_t)),\eta_t (s_t-m_t) \rangle\\
&+ \eta_t^2\frac{L}{2}\norm{s_t -m_t}^2 \\
&= F(\eta_t (s_{1:t-1}+m_t))-F(\eta_t (s_{1:t-1})) +
F(\eta_t (s_{1:t-1})) \\
&+ \langle \nabla F(\eta_t (s_{1:t-1}+m_t)),-\eta_t m_t \rangle
+\eta_t^2\frac{L}{2}\norm{s_t -m_t}^2  \\
&\leq 
F(\eta_t (s_{1:t-1})) +\eta_t^2\frac{L}{2}\norm{s_t -m_t}^2 \mbox{ (by convexity)}\footnotemark \\
&\leq \frac{\eta_t}{\eta_{t-1}}F(\eta_{t-1}s_{1:t-1})+\eta_t^2\frac{L}{2}\norm{s_t -m_t}^2 \mbox{ (by Lemma 1)}.
\end{align*} 
\footnotetext{Using the subgradient inequality $F(x)-F(y) \leq \langle \partial F(x), x-y \rangle$.}

We now proceed by induction. Observe that
$F(\eta_0 s_0) = F(0) \leq 0$ by assumption. So,
for any $\eta_0 \leq \eta_1$,\footnote{$\eta_0$ is
not used to construct $x_1$ and is only used for
the analysis.}
\[
F(\eta_1 s_{1:1})  \leq \frac{\eta_1}{\eta_0}F(\eta_0 s_0) + \eta_1^2\frac{L}{2}\norm{s_1 -m_1}^2 \leq \eta_1^2\frac{L}{2}\norm{s_1 -m_1}^2.
\]
Assume that 
\[
F(\eta_{t-1}s_{1:t-1})\leq \frac{L}{2}\sum_{k=1}^{t-1}\eta_{t-1} \eta_k \norm{s_k - m_k}^2.
\]
Then we have 
\begin{align*}
      F(\eta_t s_t) &\leq \frac{\eta_t}{\eta_{t-1}}F(\eta_{t-1}s_{t-1})+\eta_t^2\frac{L}{2}\norm{r_t -m_t}^2 \\
      &\leq \frac{L}{2}\sum_{s=1}^{t-1}\eta_{t} \eta_s \norm{r_s - m_s}^2 + \eta_t^2\frac{L}{2}\norm{r_t -m_t}^2
\end{align*}
\end{proof}

Taking no stepsize or constant stepsize and setting $m_t =0$
recovers the original results by~\citeauthor{gordon2007no}.
When $p=1$ and $m_t=0$, and applying the assumed upper bound on $s_t$, Theorem \ref{thm-opt-step} gives
\begin{align}
     R^T_{\decisionSet} \leq D\left(\frac{L C\sum_{t=1}^T\eta_t }{2B} +\frac{A}{B\eta_T}\right). 
\end{align}
Therefore, taking $\eta_t \in O(\frac{1}{\sqrt{t}})$
gives a regret bound holding uniformly over time that is of the order $O(\sqrt{T})$.

For the case of when $p>1$ where no stepsize is needed
the following regret bound is immediate

\begin{align}
    R^T_{\decisionSet}\leq D\left(\frac{L(\sum_{t=1}^T \norm{s_t-m_t}^2)+2A}{2B}\right)^{1/p}. \label{lh-opt-regret}
\end{align}

In the case of regret-matching on the simplex, where $B=1$, $A=0$, $L=2$, $D=1$, and $p=2$~\cite{gordon2007no}, we
get
\begin{align}
    R^T_{\decisionSet}\leq \sqrt{\sum_{t=1}^T \norm{s_t-m_t}^2}. \label{lh-opt-rm}
\end{align}

\subsection{Adaptive Stepsizes}

For the case of $p=1$ we can still achieve a regret 
bound similar to (\ref{lh-opt-rm}) by taking
adaptive stepsizes. Intuitively, the stepsizes 
account for how well previous predictions have 
done, or in the case of no predictions, how large
in norm $s_t$ have been.

Unlike the typical adaptive stepsize scheme for mirror
descent (\ref{ada-MD}), the stepsizes 
for Lagrangian hedging will be similar
to adaptive FTRL methods~\cite{mohri2016accelerating}, including the 
initial stepsize $\eta_1$,
\begin{align}
\label{ada-LH}
\eta_t = \frac{1}{\sqrt{\frac{1}{\eta_1^2}+\sum_{k=1}^{t-1} \norm{s_k - m_k}^2}} \quad t>1.
\end{align}

Our result is a direct application of the following Lemma,
which is a slight modification of a similar result by \citeauthor{orabona2019modern}~\cite{orabona2019modern}[Lemma 4.13], we provide the proof
in the appendix.
\begin{lemma}
Let $a_0 \geq 0$ and $0 \leq a_i \leq C$ for $i>0$. If $f$ is a non-negative decreasing function
then
\[
    \sum_{t=1}^T a_t f(a_0 + \sum_{i=1}^{t-1}a_i) \leq 
    (C -a_0)f(a_0) +  \int_{a_0}^{s_{T-1}}f(x)dx.
\]
\end{lemma}

Following adaptive stepsize scheme (\ref{ada-LH})
yields the following regret bound.
\begin{theorem}
An optimistic Lagrangian hedging algorithm with
a convex potential function $F$ satisfying conditions
(3-5), with $p=1$ and stepsizes 
\[
\eta_t = \frac{1}{\sqrt{\frac{1}{\eta_1^2}+\sum_{k=1}^{t-1} \norm{s_k - m_k}^2}} \quad t>1,
\] 
and $\eta_1 \leq \sqrt{\frac{3}{C}}$,
attains the following regret bound
\[
    R^T_{\decisionSet} \leq \frac{D}{B}\left((L+A)\sqrt{\frac{1}{\eta_1^2}+\sum_{t=1}^{T-1} \norm{s_t - m_t}^2}\right).
\]
\end{theorem}
See appendix for proof.

\subsection{Path Length Bound with Smooth Losses}
Optimism and adaptivity have found useful applications in improving
rates for several smooth problems. For example, faster rates
in smooth games~\citep{Rakhlin13b,syrgkanis2015fast, farina2019stable, farina2019optimistic}, and faster rates for offline optimization~\cite{cutkosky2019anytime,joulanisimpler}.

Unfortunately, these results strongly depend on the regret bound
having the same form as OMD and OFTRL.
However, in Lagrangian hedging we can attain 
a path length 
regret bound when the loss is fixed and smooth;
the regret is upper bounded by the change in iterates.

The new path length bound is a direct application of the general 
optimistic results of the previous section combined with the 
assumption of a fixed Lipschitz continuous smooth convex loss (possibly non-linear), 
that is 
$\ell_t =\ell$,$K \geq \norm{\nabla \ell(x)}$, and
$\norm{\nabla \ell(x) -\nabla \ell(y)}\leq L \norm{x - y}_\ast$.
If we take the typical martingale prediction $m_t = s_{t-1}$ then we have that $\norm{s_t -m_t}^2 \leq \tilde{C}\norm{x_t-x_{t-1}}_\ast^2$.
\begin{proof}

In the fixed loss case we have 
$s_t = \langle \nabla \ell(x_t), x_t \rangle u - \nabla \ell (x_t)$.
Therefore, 
\begin{align*}
    &\norm{s_t -m_t} = \norm{s_t -s_{t-1}} = \\
    &\norm{\nabla \ell (x_{t-1}) - \nabla \ell (x_t) + \langle \nabla \ell(x_t), x_t \rangle u  - \langle \nabla \ell(x_{t-1}), x_{t-1} \rangle u }\\
    &\leq L\norm{x_{t-1} - x_t}_\ast 
    + \norm{u}|\langle \nabla \ell(x_t),x_t \rangle - \langle \nabla \ell(x_{t-1}), x_{t-1} |\\
    &= L\norm{x_{t-1} - x_t}_\ast \\
    &+ \norm{u}|\langle \nabla \ell(x_t)-\nabla \ell(x_{t-1}),x_t \rangle + \langle \nabla \ell(x_{t-1}),x_t - x_{t-1})|\\
    &\leq \left(L+\norm{u}DL+ K\right)\norm{x_t-x_{t-1}}_\ast.
\end{align*}
Taking $\tilde{C} = \left(L+\norm{u}DL+ K\right)^2$ gives the result. 
\end{proof}

\section{Generalization to $\Phi$-Regret in Experts}

When $\decisionSet = \Delta^n$, the $n$-dimensional simplex,
online linear optimization becomes a problem of learning with 
expert advice. At each round $t$ an iterate $x_t \in \Delta^n$ is
a distribution over $n$ actions,
interpreted as weightings over recommendations by
$n$ experts. 
Similar to before, regret will compare the total loss with the best
$x^\ast \in \decisionSet$. However, this is equal to comparing with
the best action (best expert recommendation) and is referred to as external regret,
\begin{align}
    R^T_\decisionSet = \max_{a \in \actionSet}\sum_{t=1}^T \langle \ell_t, x_t\rangle - \langle \ell_t, \delta_a\rangle .
\end{align}
Where $\delta_a$ is a distribution over $\actionSet$ with full weight on action 
$a$. The regret can be interpreted as considering an alternative sequence of iterates
$\{\tilde{x}_t\}_{t\leq T}$, where each $\tilde{x}_t = \phi(x_t)$, for some 
transformation of the form $\phi(x) = \delta_a$. 
More generally, we can measure regret with respect to a 
set of linear transformations $\Phi$, referred to as $\Phi$ regret
\begin{align}
   R^T_\Phi = \max_{\phi \in \Phi}\sum_{t=1}^T \langle \ell_t, x_t\rangle - \langle \ell_t, \phi(x_t))\rangle.
\end{align}
Similar to Lagrangian hedging, we seek to force a vector
to some safe set. More precisely, we consider the 
$\Phi$ regret vector $s_{1:t}^\Phi$ that keeps track of how
an algorithm is doing with respect to the set $\Phi$,
\begin{align}
    s_{1:t}^\Phi = s_{1:t-1}^\Phi + s_t^\Phi.
\end{align}
Where $s_t^\Phi = \{\langle \ell_t,x_t\rangle - \langle \ell_t,\phi(x_t)\rangle\}_{\phi \in \Phi} \in \reals^{|\Phi|}$. 
If $s_{1:t}^{\Phi}$ has all non-positive entries then
$R^T_{\Phi} \leq 0$, therefore  the
safe set is chosen to be $\reals^{|\Phi|}_{\leq 0}$, the negative orthant.
This $\Phi$ regret framework, though abstract,
includes other interesting forms of regret such
as internal and swap regret~\cite{greenwald2006bounds}.
Internal regret is interesting as it allows
for efficient computation of a correlated equilibrium in game theory~\cite{cesa2006prediction}.

Similar to Lagrangian hedging, and as proposed
by \citeauthor{greenwald2006bounds}, 
the algorithms
will use a potential function $F$ to measure how 
far $s_{1:t}^{\Phi}$ is from the safe set and
slow down its growth with the Blackwell condition  
\begin{align}
    \langle \nabla F(s^\Phi_{1:t-1}), s_t^\Phi \rangle \leq 0\label{phi-blackwell}.
\end{align}
As shown by \citeauthor{greenwald2006bounds}, 
the generalized Blackwell condition with respect to $\Phi$ is achieved if 
an algorithm plays a fixed point of a
linear operator $ M_t^{\Phi}$,
\begin{align}
    M_t^{\Phi}(x) = \frac{ 
    \sum_{\phi \in \Phi} (\nabla F(s^\Phi_{1:t-1}))_{\phi} \phi(x) }{\langle \nabla F(s^\Phi_{1:t-1}), \mathbf{1}\rangle}
\end{align}
where  $(\nabla F(s^\Phi_{t-1}))_{\phi}$ denotes
the component of the vector $\nabla F(s^\Phi_{t-1}) \in \reals^{|\Phi|}$ 
associated with the transformation $\phi \in \Phi$,
and $\mathbf{1} = (1,\cdots, 1) \in \reals^{|\Phi|}$.
This fixed point exactly coincides with the Lagrangian 
hedging method when $\decisionSet$ is a simplex,
and $\Phi = \{\phi : \exists \, a \in \actionSet\, \,  \forall x \, \phi(x) = \delta_a\}$.
In other words, the rule (\ref{LH}) is a fixed point
of $M_t^{\Phi}(x)$ for external regret.

If an upperbound on $F$ provides
a regret bound, as in the previous 
sections, then optimistic Lagrangian
hedging can be generalized to 
the $\Phi$-regret setting, by defining
a new operator,
\begin{align}
    \tilde{M}_t^{\Phi}(x) = \frac{ 
    \sum_{\phi \in \Phi} (\nabla F(\eta_t(s^\Phi_{t-1}+m_t)))_{\phi} \phi(x) }{\langle \nabla F(\eta_t(s^\Phi_{t-1}+m_t)), \mathbf{1}\rangle}.
\end{align}
The main result is a theorem 
analogous to Theorem \ref{thm-opt-step}, except with the new 
regret vector $s^{\Phi}_{1:t}$.
\begin{theorem} \label{thm-opt-step-phi}
An optimistic Lagrangian hedging algorithm playing a fixed point of
$\tilde{M}^{\Phi}_t$, with
a convex potential function $F$ satisfying conditions
(\ref{safe}-\ref{smooth}) and positive decreasing stepsizes $0 < \eta_t \leq \eta_{t-1}$, ensures 
\[
   F(\eta_T s^{\Phi}_{1:T}) \leq \frac{L}{2} \sum_{t=1}^T\eta_T \eta_t \norm{s^{\Phi}_t-m_t}^2.
\]

\end{theorem}
The proof is identical to Theorem $\ref{thm-opt-step}$ except we use
the Blackwell condition
\[
\langle \nabla F(\eta_t (s^\Phi_{1:t-1}+m_t)), s_t^\Phi \rangle \leq 0,
\]
see appendix for more details.

To the best of our knowledge, this
results in the first set of
optimistic and
adaptive algorithms for minimizing 
internal and swap regret.

\subsection{Lagrangian Hedging+}
In this section we extend optimistic Lagrangian hedging in the
$\Phi$-regret setting to use a modified regret
vector 
\[
s^{\Phi+}_{1:t} = (s^{\Phi+}_{1:t-1} + s^{\phi}_t)^+.
\]
This modification is inspired by the regret-matching+ algorithm,
which has been successfully used to solve large two-player 
zero-sum games and play poker at an expert-level
~\citep{tammelin2014solving, tammelin2015solving, burch2017time}.
Indeed this framework generalizes regret-matching+,
beyond external regret and beyond regret-matching.

With the modified regret vector $s^{\Phi+}_{1:t}$, the safe set remains as $\reals^{|\Phi|}_{\leq 0}$,
because of the component wise inequality
\[
   s^{\Phi}_{1:t} \leq s^{\Phi +}_{1:t}.
\]
Therefore, $s^{\Phi +}_{1:t} \in \reals^{|\Phi|}_{\leq 0}$ implies $s^{\Phi}_{1:t}  \in \reals^{|\Phi|}_{\leq 0} $, and we have
\[
    R^T_\phi = \max_{\phi \in \Phi}(s^\Phi_{1:t})_\phi
    \leq \max_{\phi \in \Phi}(s^{\Phi + }_{1:t})_\phi.
\]

As one would expect, we define optimistic Lagrangian hedging+ with the operator $M_t$ but modified
to use the regret vector $s^{\Phi +}_{1:t}$ and a prediction $m_t$.
\begin{align}
    \tilde{M}_t^{\Phi+}(x) = \frac{ 
    \sum_{\phi \in \Phi} (\nabla F(\eta_t(s^{\Phi+}_{t-1}+m_t)))_{\phi} \phi(x) }{\langle \nabla F(\eta_t(s^{\Phi+}_{t-1}+m_t)), \mathbf{1}\rangle}. \label{lh-plus}
\end{align}

Like the previous section, if we can control
the growth of $F(s^{\Phi,+}_{1:t})$ and $F$ indeed
provides an upper bound to the safe set, then a
regret bound is attainable. 
However, we must make an additional assumption on $F$, for which we call \emph{positive invariant and smooth}~\cite{dorazio2020}. That is 
\[
    F((x+y)^+) \leq F(x)+\langle \partial F(x),y\rangle + \frac{L}{2}\norm{y}^2.
\]
Once again, equipped with the new smoothness condition
and the Blackwell condition
\[
\langle \nabla F(\eta_t (s^{\Phi +}_{1:t-1}+m_t)), s_t^\Phi \rangle \leq 0,
\]
which is guaranteed by playing the fixed point (\ref{lh-plus}) ~(see appendix for details), we have the following bound on $F$.

\begin{theorem} \label{thm-opt-step-phi-plus}
An optimistic Lagrangian hedging+ algorithm playing a fixed point of
$\tilde{M}_t^{\Phi+}$, with
a convex potential function $F$ that is positive invariant and smooth  and satisfying condition
(\ref{safe}), with positive decreasing stepsizes $0 < \eta_t \leq \eta_{t-1}$, ensures 
\[
   F(\eta_T s^{\Phi+}_{1:T}) \leq \frac{L}{2} \sum_{t=1}^T\eta_T \eta_t \norm{s^{\Phi}_t-m_t}^2.
\]
\end{theorem}
\subsection{$\Phi$ Regret Examples }
Inspired by Lagrangian hedging,
\citeauthor{greenwald2006bounds} present
different potential functions that are appropriate
to minimizing $\Phi$ regret. The functions
include a polynomial family of algorithms, with regret-matching as special case, and an exponential variant which amounts to the hedge
algorithm when external regret is minimized.

\subsubsection{Polynomial}
\begin{align}
    F(x) = \norm{x^+}^2_p \quad p \geq  2,
\end{align}
$F$ is smooth with $L = 2(p-1)$, and with respect
to the $p$-norm $\norm{\cdot}_p$.
\citeauthor{greenwald2006bounds} showed that an upper bound
on $F(s_{1:T}) \leq K$ amounts to
the regret bound $R^T_\Phi \leq \sqrt{K}$. In
the case of when an algorithm is using the modified
regret vector $s_{1:T}^{\Phi+}$ it is easy to show
$\max_{\phi \in \Phi}(s^{\Phi + }_{1:t})_\phi \leq \sqrt{F(s^{\Phi + }_{1:t})} \leq \sqrt{K}$, since
$F$ is positive invariant and smooth because $F(x^+) = F(x)$.

When $p=2$ and $\Phi$ is taken to be equivalent to external regret, we have the gradient of $F(x)$ is $x^+$, which gives the regret-matching algorithm
when if $m_t=0$, and the regret-matching+ algorithm if
$s^{\Phi +}_{1:t}$ is used with $m_t=0$. 
Notice that we exactly recover the 
regret-matching bound (\ref{lh-opt-rm}) in this case by applying the upper bound from
Theorem \ref{thm-opt-step-phi} and with no stepsize ($\eta_t = 1$).

\citeauthor{greenwald2006bounds} also
showed that the polynomial case can be extended
to $1 < p <2$ with the potential function
\[
    F(x) = \norm{x^+}^p_p \quad 1 < p < 2.
\]
However, the smoothness condition (\ref{smooth})  must be modified by replacing $\norm{\cdot}^2$ with $\norm{\cdot}^p$. This does not change the analysis but the bounds need to be changed accordingly. 
More importantly the regret bound degrades as
$p$ approaches $1$, $R^T_\Phi \leq K^{1/p}$.

Similar to the case of $p \geq 2$, if $1 < p <2$ bounds 
on $\max_{\phi \in \Phi}(s^{\Phi +}_{1:t})_\phi$ are attainable since
$F(x^+) =F(x)$ and hence is positive invariant and smooth.\footnote{See \cite{dorazio2020} for more details.}

\subsubsection{Exponential}
In addition to the polynomial family, 
we can pick the exponential variant with
potential function
\[
    F(\eta x) = \ln \sum_i e^{\eta x_i} - \ln(d).
\]
Where $x\in\reals^d$, $L=1$, and $\norm{\cdot}^2 = \norm{\cdot}^2_\infty$ for the smoothness condition.
It can also be shown that $\max_i x_i$ for some vector $x \in \reals^d$ is upper bounded by 
$\frac{1}{\eta}(F(\eta x) + \ln(d))$,
therefore the bound on $F$ from Theorem \ref{thm-opt-step-phi}
gives an upper bound on $\max_\phi (s^{\Phi}_{1:t})_{\phi}$
with $d = |\Phi|$.

The gradient of $F(x)$ is the softmax function
and gives the hedge algorithm if $\Phi$  is chosen to correspond
with external regret.

\section{Related Work}
An important instance of Lagrangian hedging is regret-matching, an algorithm typically used
within the game theory community~\citep{Hart00Rm,zinkevich2008regret},
and is a special case of Blackwell's algorithm~\cite{blackwell1956analog}.
At the same time of writing ~\citeauthor{farina2020faster} have also analyzed
an optimistic variant of regret-matching
and its popular variant regret-matching+,
named predictive regret-matching and predictive regret-matching+, respectively~\cite{farina2020faster}. 
On the surface, our analysis provides more generality as it includes both of their variants of regret-matching and more. 
However, it is conceivable that the main tool used 
in their paper, the equivalence of Blackwell approachability and online linear optimization~\cite{abernethy2011blackwell}, provides
generality to analyze other optimistic 
Blackwell style algorithms.
More importantly, we do not believe that the tools
from \citeauthor{abernethy2011blackwell} to be equivalent to Lagrangian hedging. Further investigation is left to future work.

\section{Conclusion}

In this paper we extend Lagrangian hedging to include optimism, a guess $m_t$ of how
the regret vector will change, and adaptive stepsizes.
The regret bounds attained for optimistic and adaptive Lagrangian hedging
lead to a path length bound in constrained smooth convex optimization.
Furthermore, we devise optimistic and adaptive algorithms to minimize $\Phi$ regret,
a generalization of external regret that includes internal regret, and
include a new class of algorithms that generalizes regret-matching+.

The analysis in this paper provides new algorithms, experimental
evaluation is left to future work. For example, do the new optimistic 
and adaptive algorithms for internal regret provide better convergence to 
correlated equilibria then their non-optimistic counterparts?
Additionally, in this work the step size scheme (\ref{ada-LH}) is prescribed
for potential functions with parameter $p=1$, which amounts to a new step size
scheme for the well-known hedge algorithm, with many preexisting adaptive variants, 
does this scheme provide any benefits
over other adaptive schemes in practice?

\bibliography{ref}

\appendix
\section{Appendix}

\section{Proof of Lemma 2}

Let $a_0 \geq 0$ and $0 \leq a_i \leq C$ for $i>0$ $f$. If $f$ is a non-negative decreasing function
then
\[
    \sum_{t=1}^T a_t f(a_0 + \sum_{i=1}^{t-1}a_i) \leq 
    (C -a_0)f(a_0) +  \int_{a_0}^{s_{T-1}}f(x)dx
\]
\begin{proof}
 Let $s_t =\sum_{i=0}^ta_i$.
 \begin{align*}
     a_t f(a_0 + \sum_{i=1}^{t-1}a_i) &= a_t f(s_{t-1}) = (a_t -a_{t-1})f(s_{t-1}) + a_{t-1}f(s_{t-1}) \\
     &= (a_t -a_{t-1})f(s_{t-1}) + \int_{s_{t-2}}^{s_{t-1}}f(s_{t-1})dx \\
     &\leq (a_t -a_{t-1})f(s_{t-1}) + \int_{s_{t-2}}^{s_{t-1}}f(x)dx
 \end{align*}
 So we have that
 \[
 \sum_{t=1}^T a_t f(a_0 + \sum_{i=1}^{t-1}a_i) \leq \sum_{t=1}^T (a_t-a_{t-1})f(s_{t-1}) + \int_{a_0}^{s_{T-1}}f(x)dx.
 \]
 Now we will apply the summation by parts formula to analyze the first sum.
 \begin{align*}
     &\sum_{t=1}^T (a_t-a_{t-1})f(s_{t-1})\\
     &= f(s_{T-1})a_T -f(a_0)a_0
     -\sum_{t=2}^T a_{t-1}(f(s_{t-1})-f(s_{t-2})) \\
     &= f(s_{T-1})a_T -f(a_0)a_0 + \sum_{t=1}^{T-1} a_{t}(f(s_{t-1})-f(s_{t})) \\
     &\leq f(s_{T-1})a_T -f(a_0)a_0 + \sum_{t=1}^{T-1} C(f(s_{t-1})-f(s_{t})) \\
     &= f(s_{T-1})a_T -f(a_0)a_0 +C f(a_0) - Cf(s_{T-1}) \\
     &\leq (C -a_0)f(a_0)
 \end{align*}
 The first inequality is due to $f$ being a decreasing function, hence $f(s_{t-1}) - f(s_t) \geq 0$, and because $0 \leq a_i \leq C$.
 The last inequality also follows because $a_T \leq C$.
\end{proof}

\section{Proof of Theorem 2}

An optimistic Lagrangian hedging algorithm with
a convex potential function $F$ satisfying conditions
(1-3), with $p=1$ and stepsizes 
\[
\eta_t = \frac{1}{\sqrt{\frac{1}{\eta_1^2}+\sum_{k=1}^{t-1} \norm{s_k - m_k}^2}} \quad t>1,
\] 
and $\eta_1 \leq \sqrt{\frac{3}{C}}$,
attains the following regret bound
\[
    R^T_{\decisionSet} \leq \frac{D}{B}\left((L+A)\sqrt{\frac{1}{\eta_1^2}+\sum_{t=1}^{T-1} \norm{s_t - m_t}^2}\right).
\]
\begin{proof}
 Recall assumption (\ref{dist-bound}) with $p=1$, and inequality
 (\ref{reg-dist}), then a non-negative upperbound on $F(\eta_T s_{1:T}) \leq K$ translates into the regret bound
 \[
 R^T_{\decisionSet} \leq D\left(\frac{ K}{B \eta_T}+\frac{A}{B \eta_T} \right).
 \]
 From Theorem \ref{thm-opt-step} we have that
 \[
 0 \leq K =  \frac{L}{2} \sum_{t=1}^T\eta_T \eta_t \norm{s_t-m_t}^2,
 \]
 is a valid upper bound.
 Therefore
 \[
 R^T_{\decisionSet} \leq \frac{D}{B}\left( \frac{L}{2} \sum_{t=1}^T\eta_t \norm{s_t-m_t}^2+\frac{A}{\eta_T} \right).
 \]
 We now apply  Lemma 2 on the sum across $T$ rounds
 by noticing that $\eta_t = f(a_0 +\sum_{i=1}^{t-1}a_i)$,
 where $a_i = \norm{s_i -m_i}^2$, $f(x) = \frac{1}{\sqrt{x}}$,
 and $a_0 = \frac{1}{\eta_1^2}$.
 By assumption we also have that $0 \leq a_i= \norm{s_i -m_i}^2 \leq C $.
 
 Therefore, by Lemma 2
 \begin{align*}
     &\sum_{t=1}^T\eta_t \norm{s_t-m_t}^2 \leq \\
     &\left(C - \frac{1}{\eta_1^2}\right)\eta_1 + 2\left(\sqrt{\frac{1}{\eta_1^2}+\sum_{t=1}^T \norm{s_t-m_t}^2} - \frac{1}{\eta_1}\right)\\
     &= C\eta_1 -\frac{3}{\eta_1} + 2\sqrt{\frac{1}{\eta_1^2}+\sum_{t=1}^T \norm{s_t-m_t}^2}\\
     &\leq 2\sqrt{\frac{1}{\eta_1^2}+\sum_{t=1}^T \norm{s_t-m_t}^2} \quad \mbox{ if } \eta_1 \leq \sqrt{\frac{3}{C}}.
 \end{align*}

\end{proof}

\section{Proof of Theorem 4}
An optimistic Lagrangian hedging+ algorithm playing a fixed point of
$\tilde{M}^{\Phi+}$, with
a convex potential function that is positive invariant and smooth $F$ and satisfying conditions
(\ref{safe}-\ref{smooth}), with positive decreasing stepsizes $0 < \eta_t \leq \eta_{t-1}$, ensures 
\[
   F(\eta_T s^{\Phi+}_{1:T}) \leq \frac{L}{2} \sum_{t=1}^T\eta_T \eta_t \norm{s^{\Phi}_t-m_t}^2.
\]
\begin{proof}
The proof resembles closely to that of Theorem \ref{thm-opt-step}. 

\begin{align*}
    &F(\eta_t s^{\Phi+}_{1:t}) =
    F(\eta_t (s^{\Phi+}_{1:t-1} + s^\Phi_t +m_t -m_t)^+) \\
    &\leq F(\eta_t (s^{\Phi+}_{1:t-1}+m_t)) +
    \langle \nabla F(\eta_t(s^{\Phi+}_{1:t-1}+m_t)), \eta_t(s^\Phi_t-m_t)\rangle \\
    &+ \eta_t^2\frac{L}{2}\norm{s^\Phi_t-m_t}^2 \\
    &\leq F(\eta_t (s^{\Phi+}_{1:t-1}+m_t))
    - F(\eta_t s^{\Phi+}_{1:t-1})+F(\eta_t s^{\Phi+}_{1:t-1})\\
    &+
    \langle \nabla F(\eta_t(s^{\Phi+}_{1:t-1}+m_t)), -\eta_t m_t\rangle + \eta_t^2\frac{L}{2}\norm{s^\Phi_t-m_t}^2.
\end{align*}
The rest follows from the same arguments as
Theorem \ref{thm-opt-step}.
\end{proof}

\section{The $\Phi$ Regret Fixed Point}

\citeauthor{greenwald2006bounds} showed that 
$x_t$ that is a fixed point of $M_t$ is guaranteed
to satisfy the generalized Blackwell condition
\[
\langle \nabla F(s^\Phi_{1:t-1}), s_t^\Phi \rangle \leq 0.
\]
Our results reuse this observation by modifiying the
regret vector $s^\Phi_{1:t-1}$ with a prediction
and possibly using the modified $s^{\Phi+}_{1:t-1}$
vector.
More generally, we use the following result that follows directly from \citeauthor{greenwald2006bounds}; the following inequality holds
\[
\langle \nabla F(z), s_t^\Phi \rangle \leq 0,
\]
if the operator used to construct the fixed point 
is defined to be
\[
M_t^{\Phi}(x) = \frac{ 
    \sum_{\phi \in \Phi} (\nabla F(z))_{\phi} \phi(x) }{\langle \nabla F(z), \mathbf{1}\rangle}.
\]

\end{document}